\documentclass[]{fairmeta}

\usepackage[utf8]{inputenc}
\usepackage[T1]{fontenc}
\usepackage{tcolorbox}               % no options here — avoids clash
\tcbuselibrary{breakable, skins}     % load libraries separately
\usepackage{geometry}
\usepackage{hyperref}
\usepackage{algorithm}
\usepackage{algorithmic}
\usepackage{enumitem}                 % for [nosep] lists
\usepackage{booktabs}                 % for \toprule/\midrule/\bottomrule
\usepackage{tabularx}                 % for flexible-width tables
\newcolumntype{Y}{>{\raggedright\arraybackslash\hyphenpenalty=10000\exhyphenpenalty=10000}X}

\newcommand{\method}{MaD-RL}  
\newcommand{\pr}{\mathsf{P}}
\newcommand{\mix}{\mathsf{M}}
\usepackage{amssymb,bm,amsthm,booktabs,xcolor}
\usepackage{natbib} % for \citep
\newtheorem{theorem}{Theorem}

\theoremstyle{remark}

\theoremstyle{plain}

\usepackage{amsmath,amsfonts,bm}

\definecolor{metablue_box}{RGB}{232, 240, 254}
\definecolor{metablueborder}{RGB}{66, 133, 244}
\definecolor{metagray}{RGB}{245, 245, 248}
\definecolor{metagrayborder}{RGB}{180, 180, 190}
\definecolor{mustgreen}{RGB}{26, 127, 55}
\definecolor{mustnotred}{RGB}{179, 38, 30}

\newtcolorbox{constitutionbox}[1][]{
  enhanced,
  colback=metablue_box,
  colframe=metablueborder,
  breakable,
  arc=2mm,
  boxrule=0.5pt,
  left=8pt, right=8pt, top=8pt, bottom=8pt,
  title={#1},
  fonttitle=\bfseries\small,
  coltitle=black,
  attach boxed title to top left={yshift=-2mm, xshift=4mm},
  boxed title style={colback=metablue_box, colframe=metablueborder, arc=1mm, boxrule=0.4pt},
}

\newtcolorbox{principlebox}[1][]{
  enhanced,
  colback=metagray,
  colframe=metagrayborder,
  breakable,
  arc=2mm,
  boxrule=0.4pt,
  left=8pt, right=8pt, top=8pt, bottom=8pt,
  title={#1},
  fonttitle=\bfseries\small,
  coltitle=black,
  attach boxed title to top left={yshift=-2mm, xshift=4mm},
  boxed title style={colback=metagray, colframe=metagrayborder, arc=1mm, boxrule=0.4pt},
}

\newtcolorbox{testbox}{
  enhanced,
  colback=white,
  colframe=metablueborder,
  boxrule=0.4pt,
  arc=1mm,
  left=6pt, right=6pt, top=4pt, bottom=4pt,
}

\title{MaD-RL: Matching Distributions for Calibrating LLMs with Reinforcement Learning}
\author[*]{Sourabh Kulkarni}
\author[*]{Ksheeraj Sai Vepuri}
\author[*]{Basar Demir}
\author{Jason Bohrer}
\author{Emily Shen}
\author{Jianfa Chen}
\author{Nan Jiang}
\author{Ankit Jain}
\author{Harihar Subramanyam}
\author{Mannat Singh}
\author[*]{Chirag Nagpal}

\affiliation{Meta Superintelligence Labs}
\contribution[*]{Equal contribution}

\abstract{Reinforcement learning (RL) is widely used in language-model post-training to maximize rewards assigned to individual model outputs, such as scores from binary verifiers or reward models trained on human feedback. However, applications such as synthetic-data generation, fairness-related constraint satisfaction, and policy exploration require controlling the distribution of outputs across model generations rather than only maximizing expected reward. We propose a general RL-based framework for \textit{Distribution Matching} allowing matching the distribution of a latent categorical attribute of model outputs to a specified target distribution. Empirically, we demonstrate that dominant post-training recipes such as Group Relative Policy Optimization (GRPO) reduce output diversity by concentrating policy probability towards a single mode. Entropy regularization and sampling temperature can improve the spread of the distribution but have constrained effectiveness, limited to apply only in token space and toward uniform distributions. We show that prior work in this area is a specific case of Distribution Matching involving the $L_2$ divergence. We then propose reward functions for other divergences such as KL and Jensen-Shannon and motivate them with theoretical justification. Finally, we demonstrate the effectiveness of our approach on a set of experiments involving mathematical reasoning and programming.
}

\date{\today}
\correspondence{\email{sourabhkul@meta.com}, \email{ksheerajvepuri@meta.com}}
\usepackage{amsmath,amsfonts,bm}

\def\eqref#1{equation~\ref{#1}}
\def\1{\bm{1}}

\def\vx{{\bm{x}}}
\def\vy{{\bm{y}}}

\def\mX{{\bm{X}}}

\def\mZ{{\bm{Z}}}

\DeclareMathAlphabet{\mathsfit}{\encodingdefault}{\sfdefault}{m}{sl}
\SetMathAlphabet{\mathsfit}{bold}{\encodingdefault}{\sfdefault}{bx}{n}

\begin{document}

\maketitle

%========================================================================
\section{Introduction}
\label{section:intro}

Post-training language models with Reinforcement Learning (RL) typically optimizes the expected reward of an individual models generations. Standard methods assign each completion a scalar reward, such as a binary verifier score or a learned preference score, and increase the probability of outputs that receive higher rewards~\citep{christiano2017deep,stiennon2020learning,ouyang2022training,shao2024deepseekmath,lambert2024tulu}. This formulation is natural when each response can be evaluated independently, as in mathematical reasoning. Some applications, however, require control over behavior across repeated generations, such as producing a specified mixture of response languages, solution strategies, styles, or demographic attributes. Standard pointwise objectives do not directly provide this distribution-level control. \Cref{fig:overview} illustrates this distinction.

  \begin{figure*}[!t]
      \centering
      \includegraphics[width=\linewidth]{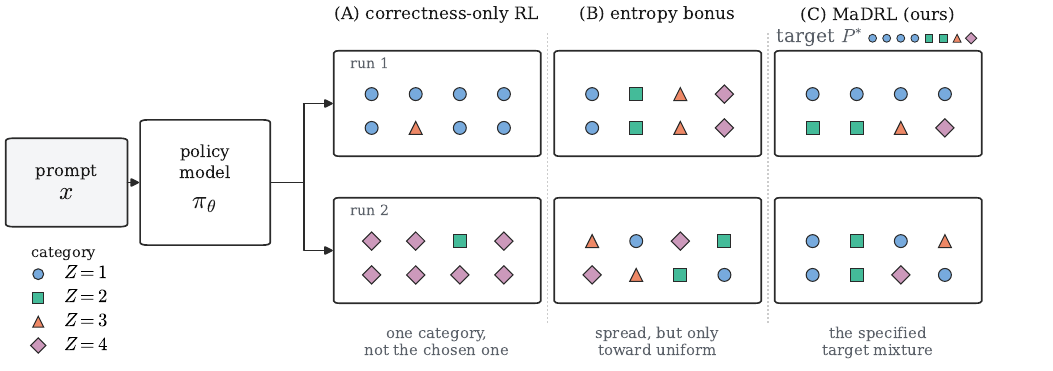}
      \caption{\textbf{Schematic comparison of the category-level objectives.}
      Eight completions sampled for one prompt, one glyph per category, under three
      post-training objectives and two independent runs. Correctness-only RL
      concentrates on a single category but does not control which one: the two
      runs settle on different ones. An entropy bonus can restore spread, but does not
      specify the desired proportions. \method{} instead targets a user-specified
      distribution including a non-uniform one (Column C).}
      \label{fig:overview}
  \end{figure*}

Recent work has begun to optimize properties of output distributions directly. \citeauthor{anschel2025group} introduce GAPO, which uses category frequencies among completions sampled for the same prompt to promote uniform coverage. GAPO relates its frequency reward to entropy maximization, but does not derive it as the policy-gradient coefficient of an explicit target-distribution divergence or compare alternative divergence objectives. \citeauthor{mohri2026generalized} present a general theoretical framework for answer-level distributional optimization and analyze finite-sample reward estimation for squared-Euclidean and KL objectives~(\citeyear{mohri2026generalized}). However, neither work empirically compares $L_2$, forward KL, reverse KL, and Jensen-Shannon (JS) objectives for fitting user-specified uniform, non-uniform, and truncated categorical targets alongside a task reward.
% \newpage
To address this gap, we introduce \emph{\underline{Ma}tching \underline{D}istributions with \underline{R}einforcement \underline{L}earning} (\method{}), a GRPO-compatible method that fits user-specified categorical target distributions while optimizing a task reward. For each prompt, \method{} samples multiple completions, assigns them to predefined categories, and estimates the frequency of each category. It then assigns each completion a distributional reward derived from $L_2$, forward KL, reverse KL, or Jensen--Shannon divergence. We study uniform and non-uniform targets, including peaked targets that assign zero probability to selected categories. Our contributions are:
\begin{itemize}\item We introduce \method{}, which combines task rewards with
distributional feedback from sampled category frequencies. We derive GRPO-compatible
rollout rewards from $L_2$, forward KL, reverse KL, and JS objectives.

\item We systematically compare these objectives in controlled settings with synthetic experiments, as well as experiments involving multilingual mathematical reasoning, and multi-language programming. In experiments with non-uniform targets, the KL and JS-based objectives achieve closer target adherence than the $L_2$ objective.

\item We show that \method{} steers category frequencies toward specified targets and quantify the associated trade-off in task accuracy. Our evaluations measure performance across natural languages in the mathematical reasoning and programming languages in code generation.

\end{itemize}

\section{Matching Distributions with RL (MaDRL)}
\label{section:method}

For each prompt, \method{} samples a group of completions, assigns each one to a
category, compares the group frequencies with the target, and combines the resulting
distributional reward with task correctness in GRPO. We now formalize this procedure.

\subsection{Formalization}

Let $\mX\sim\mathcal D$ be a prompt and
$\vy\sim\bm{\pi}_{\bm\theta}(\cdot\mid \mX)$ a generated completion. A fixed function
$f:\mathcal Y\rightarrow\mZ$ assigns the completion to one
of $K$ categories, $\mZ=\{1,\ldots,K\}$. For $z\in\mZ$, the policy induces
\begin{equation}
\pr_{\pi}(\mZ = z\mid \mX = \vx)
:=
\Pr_{\vy\sim\bm{\pi}_{\bm\theta}(\cdot\mid \vx)}[f(\vy)=z].
\label{eq:category-distribution-main}
\end{equation}
The derivations assume that every output belongs to one of the $K$ target categories,
so $\pr_{\pi}(\cdot\mid \mX = \vx)$ sums to one. In experiments with off-support outputs,
category counts retain the full group size as their denominator and those outputs
receive a fixed penalty. The resulting updates approximate the population objective below.

Let $\pr^*(\cdot\mid \mX = \vx)$ be the desired category distribution. For compactness,
write $\pr^*_{\vx}=\pr^*(\cdot\mid \mX = \vx)$ and
$\pr_{\pi,\vx}=\pr_{\pi}(\cdot\mid \mX = \vx)$. Given a discrepancy $D$, distribution
matching minimizes
\begin{equation}
\begin{aligned}
\bm\theta_D^\star
&=\operatorname*{arg\,min}_{\bm\theta}\mathcal L_D(\bm\theta), \\
\mathcal L_D(\bm\theta)
&=\mathop{\mathsf{E}}_{\mX\sim\mathcal D}
\left[D\!\left(\pr^*_{\mX},\pr_{\pi,\mX}\right)\right].
\end{aligned}
\label{eq:distribution-objective-main}
\end{equation}
The target may depend on the prompt, although all of our experiments use the same
target for every prompt.

To derive a rollout reward for this objective, fix a prompt and write
$\pr^*(z)=\pr^*(z\mid \mX = \vx)$ and $\pr_{\pi}(z)=\pr_{\pi}(z\mid \mX = \vx)$. The
score-function identity gives
\begin{equation}
\nabla_{\bm\theta} \pr_{\pi}(z)
=
\mathop{\mathsf{E}}_{\vy\sim\bm{\pi}_{\bm\theta}(\cdot\mid \vx)}
\left[
\mathbf 1\{f(\vy)=z\}
\nabla_{\bm\theta}\log\bm{\pi}_{\bm\theta}(\vy\mid \vx)
\right].
\label{eq:category-score-identity}
\end{equation}
This identity turns a distribution objective into a rollout-level reward. In
particular, define
\begin{equation}
\texttt{R}_D(z,\vx)
=
-\left.
\frac{\partial D(\pr^*(\cdot\mid \mX = \vx),\pr_{\pi})}{\partial \pr_{\pi}(z)}
\right|_{\pr_{\pi}=\pr_{\pi}(\cdot\mid \mX = \vx)}
+b_D(\vx),
\label{eq:population-coefficient}
\end{equation}
where $b_D(\vx)$ is independent of the category. Since
$\sum_z\nabla_{\bm\theta} \pr_{\pi}(z)=0$, the added term acts as a policy-gradient
baseline, and
\begin{equation}
\mathop{\mathsf{E}}_{\mX,\vy}
\left[
\texttt{R}_D(f(\vy),\mX)
\nabla_{\bm\theta}\log\bm{\pi}_{\bm\theta}(\vy\mid \mX)
\right]
=
-\nabla_{\bm\theta}\mathcal L_D(\bm\theta).
\label{eq:general-divergence-gradient-main}
\end{equation}
Here the expectation samples $\mX\sim\mathcal D$ and then
$\vy\sim\bm{\pi}_{\bm\theta}(\cdot\mid \mX)$.
We choose the baselines in \cref{table:divergence-coefficients} so that each
distributional reward is zero when $\pr^*=\pr_{\pi}$.

\subsection{GAPO and the Squared-$L_2$ Objective}

We first consider squared $L_2$, which recovers the frequency-dependent part of
GAPO~\citep{anschel2025group}.
For a prompt $\vx$, let $\vy_1,\ldots,\vy_G$ be the sampled rollout group and let
$z_i=f(\vy_i)$. The empirical frequency of category $c$ is
\begin{equation}
\widehat{\pr}_{\pi}(c)
=
\frac{1}{G}\sum_{j=1}^{G}\mathbf 1\{z_j=c\}.
\label{eq:group-frequency-main}
\end{equation}
When every rollout is valid and the target is uniform, removing the constant shared by
all rollouts leaves the GAPO reward
\begin{equation}
\widehat{\texttt{R}}^{\,\mathrm{GAPO}}_i
=
\frac{1}{K}-\widehat{\pr}_{\pi}(z_i).
\label{eq:gapo-reward-main}
\end{equation}
Replacing the uniform target with an arbitrary target gives
\begin{equation}
\widehat{\texttt{R}}_{L_2,i}
=
\pr^*(z_i\mid \mX = \vx)-\widehat{\pr}_{\pi}(z_i).
\label{eq:l2-group-reward}
\end{equation}
A category receives a positive reward when it appears less often than requested and a negative reward when it appears too often. This simple rule also has a direct optimization interpretation.

\begin{theorem}[Population squared-$L_2$ identity]
At the population level, the reward
$\texttt{R}_{L_2}(z,\vx)=\pr^*(z)-\pr_{\pi}(z)$ produces a
policy-gradient update equal to the negative gradient of
$D_{L_2}(\pr^*,\pr_{\pi})=\frac{1}{2}\sum_{z=1}^{K}(\pr_{\pi}(z)-\pr^*(z))^2$.
\end{theorem}

\noindent\emph{Proof.}
For squared $L_2$,
$-\partial D_{L_2}/\partial \pr_{\pi}(z)=\pr^*(z)-\pr_{\pi}(z)$. The result follows
directly from \cref{eq:general-divergence-gradient-main}.

The theorem uses the population probability $\pr_{\pi}(z)$. Because each rollout
contributes to its own estimate $\widehat{\pr}_{\pi}(z)$, the plug-in reward can be trivially extended to the conditional case where each prompt has its own corresponding target distribution. GRPO normalization, clipping, and task-reward mixing further separate
the implemented update from the population gradient.

\subsection{Forward and Reverse KL Objectives}

The same construction gives rewards for other divergence objectives. Because KL
divergence is asymmetric, its two directions produce different rewards. For forward
KL, $D_{\mathrm{FKL}}(\pr^*,\pr_{\pi})=D_{\textrm{KL}}(\pr^*\Vert \pr_{\pi})$, and
\begin{equation}
\texttt{R}_{\mathrm{FKL}}(z,\vx)
=
\frac{\pr^*(z)}{\pr_{\pi}(z)}-1.
\label{eq:fkl-reward}
\end{equation}
For reverse KL, $D_{\mathrm{RKL}}(\pr^*,\pr_{\pi})=D_{\textrm{KL}}(\pr_{\pi}\Vert \pr^*)$, and
\begin{equation}
\texttt{R}_{\mathrm{RKL}}(z,\vx)
=
\log \pr^*(z)-\log \pr_{\pi}(z).
\label{eq:rkl-reward}
\end{equation}
The constants in these expressions are the baseline choices from
\cref{eq:population-coefficient}. Forward KL is singular when $\pr^*(z)>0$ but
$\pr_{\pi}(z)=0$.
Reverse KL is singular when $\pr_{\pi}(z)>0$ but $\pr^*(z)=0$, so the reverse-KL
experiments use an $\varepsilon$-smoothed target for a peaked distribution to ensure numerical stability in practice.

\subsection{Jensen--Shannon Divergence}

Jensen--Shannon divergence provides a symmetric alternative by using the mixture
$\mix=(\pr^*+\pr_{\pi})/2$:
\begin{equation}
\textrm{JSD}(\pr^*,\pr_{\pi})
=
\frac{1}{2}D_{\textrm{KL}}(\pr^*\Vert \mix)
+
\frac{1}{2}D_{\textrm{KL}}(\pr_{\pi}\Vert \mix).
\label{eq:js-objective}
\end{equation}
Taking the negative derivative with respect to $\pr_{\pi}(z)$ gives
\begin{align}
\texttt{R}_{\mathrm{JSD}}(z,\vx)
&=
\frac{1}{2}\left(\log \mix(z)-\log \pr_{\pi}(z)\right) \\
&=
-\frac{1}{2}
\log\left(\frac{2\pr_{\pi}(z)}{\pr^*(z)+\pr_{\pi}(z)}\right).
\label{eq:js-coefficient}
\end{align}
We retain the factor $1/2$ from the standard JSD definition; dropping it would double
the distributional reward relative to correctness. Note that while JSD is bounded above by $\log 2$,
but the corresponding reward for policy gradient optimization is not. We thus require $\epsilon$-smoothness for the JSD reward when running RL similar to KL Divergence.

\paragraph{Connection to GAIL.}
The JSD reward is closely related to the density-ratio reward used in
GAIL~\citep{ho2016generative}. If a discriminator distinguishes categories drawn from
$\pr^*(\cdot\mid \mX = \vx)$ from those drawn from $\pr_{\pi}(\cdot\mid \mX = \vx)$, its
optimum is
$d^*(\vx,z)=\pr^*(z\mid \mX = \vx)/(\pr^*(z\mid \mX = \vx)+\pr_{\pi}(z\mid \mX = \vx))$.
Therefore, substituting the optimal discriminator into the adversarial objective gives
\begin{equation}
\max_\phi V(\phi,\bm\theta)
=
2\,\mathop{\mathsf{E}}_{\mX}
\left[
\textrm{JSD}\!\left(\pr^*_{\mX},\pr_{\pi,\mX}\right)
\right]
-\log 4.
\end{equation}
The corresponding policy reward satisfies
\begin{equation}
-\log\!\left(1-d^*(\vx,z)\right)
=
2\texttt{R}_{\mathrm{JSD}}(z,\vx)+\log 2.
\end{equation}
This is the GAIL argument applied to conditional category distributions rather than
occupancy measures. Unlike GAIL, \method{} does not need a learned discriminator
because its target is known and rollout frequencies are obtained by counting. \Cref{table:divergence-coefficients} summarizes the four population rewards.

\begin{table}[t]
    \centering
    \normalsize
    \setlength{\tabcolsep}{6pt}
    \begin{tabular}{l c c}
    \toprule
    \textbf{Divergence} & \textbf{Objective} & \textbf{Reward $\texttt{R}_D(z,\vx)$}\\
    \midrule
    $L_2$
      & $\displaystyle \frac{1}{2}\sum_d(\pr_{\pi}(d)-\pr^*(d))^2$
      & $\displaystyle \pr^*(z)-\pr_{\pi}(z)$ \\
    $D_{\textrm{KL}}(\pr^*\Vert \pr_{\pi})$
      & $\displaystyle \sum_d \pr^*(d)\log\frac{\pr^*(d)}{\pr_{\pi}(d)}$
      & $\displaystyle \frac{\pr^*(z)}{\pr_{\pi}(z)}-1$ \\
    $D_{\textrm{KL}}(\pr_{\pi}\Vert \pr^*)$
      & $\displaystyle \sum_d \pr_{\pi}(d)\log\frac{\pr_{\pi}(d)}{\pr^*(d)}$
      & $\displaystyle \log \pr^*(z)-\log \pr_{\pi}(z)$ \\
    $\textrm{JSD}(\pr^*,\pr_{\pi})$
      & $\displaystyle
         \frac{1}{2}D_{\textrm{KL}}(\pr^*\Vert \mix)
         +\frac{1}{2}D_{\textrm{KL}}(\pr_{\pi}\Vert \mix)$
      & $\displaystyle
         \frac{1}{2}\bigl(\log \mix(z)-\log \pr_{\pi}(z)\bigr)$ \\
    \bottomrule
    \end{tabular}
    \caption{\textbf{Reward functions corresponding to for the different divergence objectives.}
    Category-independent constants are chosen so each reward vanishes at
    $\pr^*=\pr_{\pi}$.}
    \label{table:divergence-coefficients}
\end{table}

\subsection{Finite-Group Training and Task Rewards}

To use these population rewards in training, we replace $\pr_{\pi}(z)$ with
$\widehat{\pr}_{\pi}(z)$ and treat the
empirical frequency as fixed during the update. For a sampled category,
$\widehat{\pr}_{\pi}(z_i)\geq1/G$, but the nonlinear KL and JSD rewards are still biased at
finite group size. An on-policy update also cannot directly reward a category that is
absent from the group.

Let $\texttt{R}_i^{\mathrm{div}}$ denote the selected empirical distributional reward. In
the reasoning and programming experiments, we combine it with binary correctness in two
ways:
\begin{align}
\texttt{R}_i^{\mathrm{additive}}
&=
\alpha \texttt{R}_i^{\mathrm{correctness}}
+(1-\alpha)\texttt{R}_i^{\mathrm{divergence}} \quad \text{and,} \label{eqn:additive}
\\
\texttt{R}_i^{\mathrm{gated}}
&=
\texttt{R}_i^{\mathrm{correctness}}
\left(1+\lambda_{\mathrm g}\texttt{R}_i^{\mathrm{divergence}}\right).\label{eqn:gated}
\end{align}
The additive form gives distributional feedback to every categorized rollout. The
gated form is a heuristic that applies it only through correct rollouts. Unlike the
additive form, it is affected by category-independent shifts because those shifts also
rescale the correctness reward. We therefore use the baseline choices in
\cref{table:divergence-coefficients}. GRPO then forms group-relative advantages and
applies its standard clipped, KL-regularized update~\citep{shao2024deepseekmath}.

The next section compares these four rewards, first in isolation and then together with
correctness.

% \newpage
\section{Experiments}
\label{section:experiments}

\begin{figure*}[!t]
    \centering
    \includegraphics[width=\textwidth]{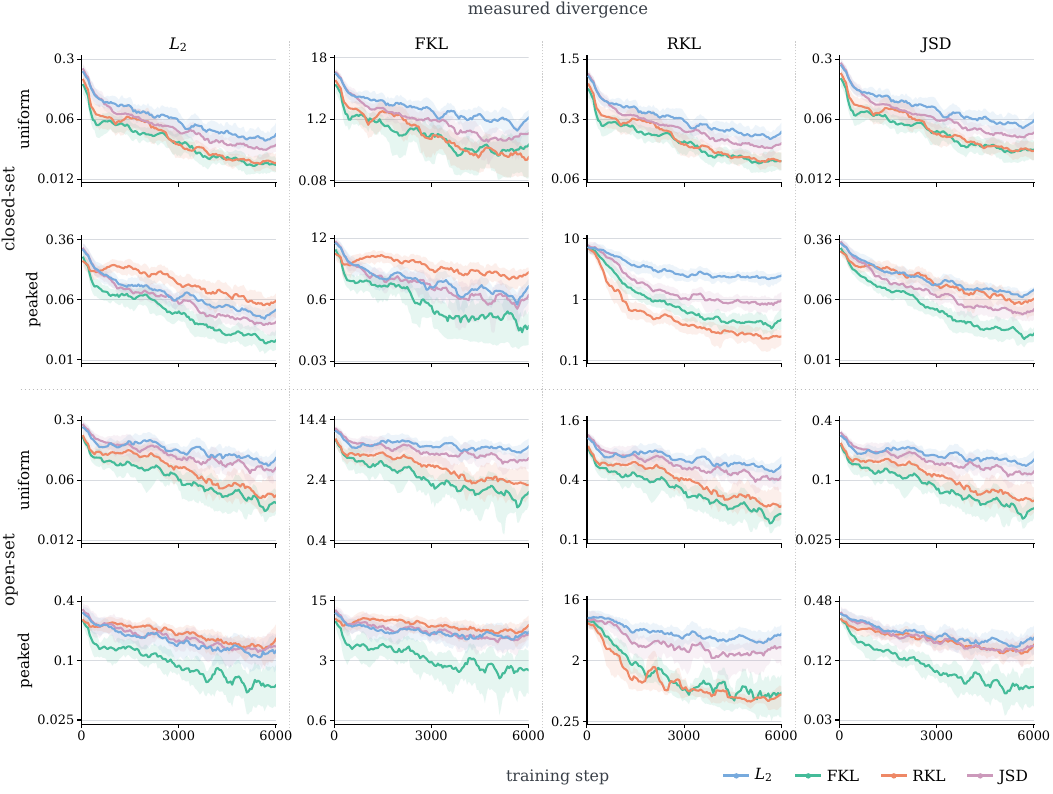}
    \caption{\textbf{Divergence across RL Steps on \textbf{SYNTHETIC}.} The upper block lists the options in the
    prompt (closed-set), the lower block does not (open-set). Rows are targets, columns
    are evaluation metrics, curves are the training objective; every run is scored under
    all four metrics. Markers are means over three seeds; bands span their minimum
    and maximum.}
    \label{fig:closedset-div}
\end{figure*}

We evaluate \method{} across three experiments. We begin with a synthetic choice experiment, run in a closed-set and an open-set form, that allows us to study distribution matching separately from task correctness. We
then move to two realistic post-training
experiments--multilingual mathematical reasoning and multi-language programming,
where we study distribution matching jointly with task correctness. All three
experiments use five mutually exclusive categories and the same two target
distributions.

\subsection{Experiments and Datasets}
\label{section:tasks-datasets}
In this section we describe the datasets and the corresponding experiments we use to benchmark \method{}.

\textbf{SYNTHETIC (closed and open-set choice)}
We construct a synthetic choice task in which each prompt asks the model to select
one answer from five options. We run it in two forms. In the \emph{closed-set} form the
prompt provides the complete set of valid answers; in the \emph{open-set} form the same
clause is removed, so the model must infer the answer set as well as spread over it. In
both, the selected answer defines the output category.
The dataset contains
$64$ common-knowledge topics with ten prompt templates per topic, such as
\emph{"Name one of the five~\ldots"} and
\emph{"Which of the five~\ldots{} comes to mind?"}, with the five options listed in a
different random order in each. For example, the oceans topic lists \emph{pacific,
atlantic, indian, southern, and arctic.} Because every listed option is correct,
category choice does not affect task correctness.

\textbf{Multilingual Mathematical Reasoning.}
We use two mathematics benchmarks: GSM8K~\citep{cobbe2021training}, grade-school
problems, and MATH~\citep{hendrycks2021measuring}, competition problems. In both, the
model must solve each problem and write the solution in one of five languages:
\emph{English, Chinese, Spanish, Hindi, or French}. Correctness is determined by extracting
the final answer from the generated solution and comparing it with the reference
answer.

\textbf{Multi-Language Programming.}
We use CodeContests~\citep{li2022competition}, a competitive-programming benchmark.
The model generates solutions in five programming languages: \emph{Python, C++, Java, Go,
and JavaScript.} Correctness is determined by executing each generated program against
the full test suite.
Together, the mathematics and programming tasks test distributional control when the
choice of category affects correctness.

\textbf{Targets and evaluation metrics.}
All three experiments use the same two target vectors:
\begin{equation}
\begin{aligned}
p^{\star}_{\mathrm{uniform}}
&= \left(\tfrac{1}{5},\tfrac{1}{5},\tfrac{1}{5},\tfrac{1}{5},\tfrac{1}{5}\right), \\
p^{\star}_{\mathrm{peak}}
&= \left(0,0,\tfrac{1}{3},\tfrac{1}{3},\tfrac{1}{3}\right).
\end{aligned}
\label{eq:target-distributions}
\end{equation}
The uniform target assigns equal probability to every category. The peaked target
assigns no mass to two of the categories and divides the mass equally among the
remaining three. On the synthetic choice task, we exclude the first two choices
in alphabetical order. In the math experiment we exclude English and Chinese; in the
programming experiment, Python and C++. The exact zeros make this a more demanding
test of how the four objectives handle categories outside the target support.

We measure distributional fit using Jensen--Shannon divergence (JSD) between the
observed category frequencies and the target; lower is better, and values are reported
in nats. For the two post-training experiments, pass@1 is the probability that a single
sampled completion is correct. For the synthetic experiment, we also score every run
under all four divergences. We report the \emph{off-support rate} as the fraction of
responses that do not match one of the five valid answers.
Together, these metrics separate distributional fit from validity and task performance.

\subsection{Model and Training Protocol}
\label{section:setup}

\textbf{Model and initialization.} All experiments use Qwen3-4B-Instruct-2507~\citep{qwen2025qwen3instruct}. Without
language-specific prompting, the base model produces almost all categorized math solutions in English and programming solutions in Python. Each post-training dataset therefore begins from its own supervised warm-start that makes all five categories reachable during sampling. Every run within a dataset uses the same initialization and remains KL-regularized to it. Evaluation prompts do not request a language, so the observed distribution must result from training rather than inference-time steering.

\textbf{Rewards and comparisons.}
For each prompt, \method{} samples a group of completions, assigns each completion to
a category, and estimates the category frequencies within the group. It then gives
underrepresented categories more credit and overrepresented categories less credit.
We compare rewards derived from four distribution objectives: $L_2$, forward KL
(FKL), reverse KL (RKL), and Jensen--Shannon divergence (JSD)
(\cref{section:method}). For valid outputs under a uniform target, the
frequency-dependent term of the $L_2$ reward is identical to GAPO's frequency
reward~\citep{anschel2025group}. Because the peaked target assigns zero probability
to two categories, the RKL experiments use an $\varepsilon$-smoothed target.
% AUTHOR INPUT REQUIRED: report the numerical epsilon, the flooring rule, and
% whether the smoothed target is renormalized.

  For the math and programming experiments, where correctness varies across
  completions, we combine the binary correctness reward with the distributional reward.
  Let $\alpha$ denote the correctness weight. The main
  results use the \textit{Additive} form as in \cref{eqn:additive} with $\alpha\in\{0.90,0.80,0.70,0.55\}$. This provides distributional
  feedback to both correct and incorrect completions.

  We also evaluate the correctness-gated form as in \cref{eqn:gated}
  which applies the distributional adjustment only to correct completions. We compare these runs with the supervised initialization
and with correctness-only reinforcement learning with verifiable rewards
(RLVR)~\citep{lambert2024tulu}, trained from the same checkpoint and with the same
budget. The full grid is reported in
\cref{appendix:rl-grid}.

\textbf{Optimization and evaluation.}
For the synthetic choice task, we train full-model \method{} for $6000$ steps with
  group size $G{=}32$, learning
  rate $5\times10^{-6}$, rollout temperature $1.0$, and KL coefficient $0.04$. We evaluate on an $80$-prompt validation set every
  $500$ steps and use a separate $80$-prompt holdout set for final evaluation. Four
  objectives, two targets, two prompt forms, and three seeds give $48$ runs.

  For the math and programming experiments, we use the same GTDO approach for $1200$ steps with group size $G{=}16$, learning
  rate $5\times10^{-6}$, rollout temperature $1.0$, and KL coefficient $0.04$. We
  evaluate on a $60$-problem validation set every $100$ steps and use a separate
  $200$-problem holdout set for final evaluation, sampling four completions per problem.
  Each configuration is run with three random seeds on each of the three datasets. For
  compactness, the main figures and \cref{table:rl-main} report, for each objective, the
  median over twelve additive runs: four correctness weights times three seeds. These
  summaries describe typical performance across the tested trade-off settings rather
  than a single tuned value of $\alpha$; \cref{appendix:rl-grid} reports each setting.
% AUTHOR INPUT REQUIRED: state the synthetic-experiment checkpoint-selection rule.

\subsection{Analysis and Results}
\label{section:analysis}

\begin{figure*}[!t]
    \noindent
    % Left: figure
    \begin{minipage}[t]{0.53\textwidth}
        \vspace{0pt}
        \centering

        \includegraphics[width=\linewidth]{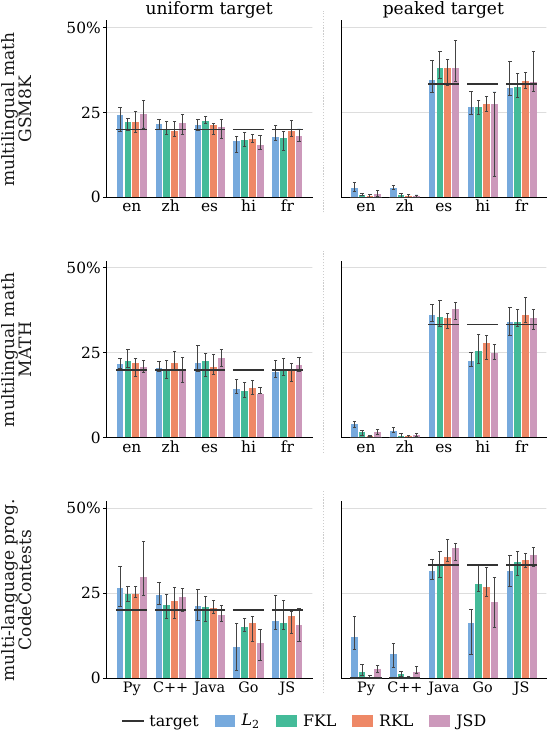}

        \captionof{figure}{\textbf{Category shares against the target.}
        Rows are the three post-training datasets --- GSM8K and MATH for
        multilingual mathematical reasoning, CodeContests for multi-language
        programming --- columns the two targets; the dash is the target share
        for each category. Whiskers are interquartile ranges. All four
        objectives fit the uniform target closely; the peaked target separates
        them; $L_2$ leaves the largest residual on all three peaked targets.}
        \label{fig:rl-results}
    \end{minipage}
    \hfill
    % Right: table
    \begin{minipage}[t]{0.43\textwidth}
        \vspace{0pt}
        \centering

        \scriptsize
        \setlength{\tabcolsep}{2pt}
        \renewcommand{\arraystretch}{1.25}
        
\centering
\footnotesize
\begin{NiceTabular}{l@{~~}c@{~~}c@{~~}c@{~~}c}
\toprule
& \multicolumn{2}{c}{uniform target} & \multicolumn{2}{c}{peaked target} \\
\cmidrule(lr){2-3}\cmidrule(lr){4-5}
Model / reward & JSD $\downarrow$ & pass@1 & JSD $\downarrow$ & pass@1 \\
\midrule
\multicolumn{5}{l}{\emph{Multilingual Math --- GSM8K}} \\[1pt]
~~base & $0.4228$ & $0.850$ & $0.6931$ & $0.850$ \\
~~SFT init & $0.0099$ & $0.804$ & $0.2278$ & $0.804$ \\
~~RLVR & $0.1184$ & $0.867$ & --- & --- \\
~~\method{} \texttt{l2} & $0.0032$ & $0.825$ & $0.0218$ & $0.817$ \\
~~\method{} \texttt{fkl} & $0.0042$ & $0.823$ & $0.0105$ & $0.802$ \\
~~\method{} \texttt{rkl} & $\bm{0.0024}$ & $0.821$ & $\bm{0.0064}$ & $0.810$ \\
~~\method{} \texttt{jsd} & $0.0048$ & $0.840$ & $0.0108$ & $0.804$ \\
\midrule
\multicolumn{5}{l}{\emph{Multilingual Math --- MATH}} \\[1pt]
~~base & $0.4228$ & $0.446$ & $0.6931$ & $0.446$ \\
~~SFT init & $0.0029$ & $0.350$ & $0.1382$ & $0.350$ \\
~~RLVR & $0.0842$ & $0.467$ & --- & --- \\
~~\method{} \texttt{l2} & $0.0062$ & $0.448$ & $0.0291$ & $0.410$ \\
~~\method{} \texttt{fkl} & $0.0073$ & $0.440$ & $0.0109$ & $0.398$ \\
~~\method{} \texttt{rkl} & $\bm{0.0055}$ & $0.444$ & $\bm{0.0065}$ & $0.381$ \\
~~\method{} \texttt{jsd} & $0.0067$ & $0.454$ & $0.0136$ & $0.402$ \\
\midrule
\multicolumn{5}{l}{\emph{Multi-language Programming --- CodeContests}} \\[1pt]
~~base & $0.4228$ & $0.392$ & $0.6931$ & $0.392$ \\
~~SFT init & $0.0007$ & $0.317$ & $0.1591$ & $0.317$ \\
~~RLVR & $0.2614$ & $0.379$ & --- & --- \\
~~\method{} \texttt{l2} & $0.0288$ & $0.352$ & $0.0949$ & $0.338$ \\
~~\method{} \texttt{fkl} & $\bm{0.0068}$ & $0.333$ & $0.0135$ & $0.317$ \\
~~\method{} \texttt{rkl} & $0.0082$ & $0.335$ & $\bm{0.0059}$ & $0.321$ \\
~~\method{} \texttt{jsd} & $0.0168$ & $0.333$ & $0.0292$ & $0.317$ \\
\bottomrule
\end{NiceTabular}

        \captionof{table}{\textbf{Distribution fit and accuracy.}
        JSD in nats, lower is better; accuracy is holdout pass@1,
        $200$ problems at four samples each. RLVR is target-agnostic and
        scored against the uniform target only. RKL fits best in five of
        the six dataset--target combinations and $L_2$ worst in four,
        at $1.2$--$8.5$ points of pass@1 below RLVR.}
        \label{table:rl-main}
    \end{minipage}
\end{figure*}

\textbf{Correctness-only RL changes the output distribution.}
The supervised initializations are close to uniform, but correctness-only RL moves away
from that distribution: the model concentrates in Chinese on GSM8K, in Spanish on
MATH, and largely returns to Python on CodeContests. Across all three datasets, every
uniform-target \method{} configuration achieves lower JSD than the best RLVR seed,
regardless of the distribution objective, reward form, mixing weight, or random seed.
This confirms that optimizing correctness alone does not preserve a desired output
mixture.

\textbf{Synthetic closed-set and open-set choice.}
\label{section:exp-closed} \Cref{fig:closedset-div} plots training-time divergence for both forms; held-out
divergence and off-support rates for every run are in \cref{table:closedset}. On held-out
prompts, FKL achieves the lowest mean divergence in all sixteen form--target--metric
combinations and $L_2$ finishes last in fourteen. On the uniform target, training with FKL, RKL, or JSD yields a
lower final divergence than training with $L_2$ under each of the four evaluation
metrics. When measured by $L_2$ divergence, FKL training reduces the final divergence by
roughly one-half on the uniform target and one-third on the peaked target relative to
$L_2$ training. For a positive-target category that becomes rare, the $L_2$ correction
remains bounded, whereas the population FKL coefficient grows as the category
probability approaches zero. FKL therefore applies a stronger correction when a rare
category is sampled; neither objective can directly recover a category that is absent
from the rollout group. Removing the option list makes the task harder in both respects:
held-out JSD rises from $0.047$--$0.200$ to $0.141$--$0.429$ and the off-support rate
from at most $0.040$ to at least $0.098$. The ordering among the objectives is unchanged,
so we next ask how these rewards behave when distribution matching is coupled with
correctness.

\textbf{Multilingual mathematical reasoning.}
\label{section:exp-realistic} Turning to the realistic tasks, we first consider the two mathematics datasets, GSM8K
and MATH. On both, all four objectives closely match the uniform target, with median
JSD below $0.005$ on GSM8K and below $0.008$ on
MATH (\cref{fig:rl-results,table:rl-main}). The peaked target is harder: RKL obtains the
lowest median JSD on both datasets ($0.0064$ and $0.0065$) and $L_2$ the largest
residual ($0.0218$ and $0.0291$). Across the two targets, pooled median pass@1 ranges
from $0.802$ to $0.840$ on GSM8K, against $0.867$ for correctness-only RLVR, and from
$0.381$ to $0.454$ on MATH, against $0.467$. On the MATH uniform target the four
objectives fall within $0.0018$ nats of each other, close enough to the finite-sample
floor that the ordering is not resolved.

\textbf{Multi-language programming.} We then turn to programming, where the mixture is tied more tightly to accuracy because
the target moves probability away from Python and C++, where the model is strongest
(\cref{fig:rl-results,table:rl-main}). FKL fits the uniform target best and RKL the
peaked target, where the spread is widest: RKL reaches $0.0059$ JSD against $L_2$'s
$0.0949$. Across both targets, pooled median pass@1 ranges from $0.317$ to $0.352$,
against $0.379$ for correctness-only RLVR.
These results illustrate the tradeoff of matching a target that moves probability away from
the model's strongest programming languages.

\textbf{Distributional objective and corresponding reward.} The ranking changes between the synthetic and post-training settings. FKL performs
best when the synthetic task isolates distribution matching. When matching is combined
with correctness, RKL achieves the lowest median JSD in five of the six dataset--target
combinations, including all three peaked targets. The experiments do not isolate which
difference causes this reversal, but they show that no divergence dominates both
settings. On the peaked target, RKL also has higher pass@1 than FKL and JSD, while
$L_2$ retains higher pass@1 at substantially poorer fit.
Neither reward form consistently performs best: additive rewards yield lower JSD on all
three peaked targets and on the programming uniform target, while the gated form is
lower on the two mathematics uniform targets by less than the finite-sample floor
(\cref{table:rl-grid}).
Overall, neither the divergence nor the reward form is uniformly best across settings.

\textbf{Accuracy Trade-offs.} Across the three post-training datasets, \method{} is $1.2$--$8.5$ percentage points
  below correctness-only RLVR in pooled median holdout pass@1. Accuracy differs across
  output languages, so changing the category mixture can also change aggregate pass@1.
  However, because each policy
  chooses which problems to answer in each language, these measurements do not cleanly
  separate mixture effects from changes in within-category behavior.

\textbf{Mode collapse across Categories.} A separate failure mode appears when a requested category becomes rare. Some runs reduce a category with positive target mass to less than $2\%$ of completions.
In an additive run, once every completion in a rollout group has the same category, the
distributional component is constant within the group and disappears under
group-relative centering, so nothing in the reward pulls the missing category back. The
supervised warm-start provides initial coverage but does not eliminate this failure
mode.
Maintaining enough exploration to recover missing categories therefore remains an
important limitation.

\section{Background and Related Work}
\label{section:related}

\textbf{Post-training and diversity.} Early preference-based RLHF learned reward models from human
comparisons~\citep{christiano2017deep}. It was later applied to summarization and
instruction following~\citep{stiennon2020learning,ouyang2022training}. These systems
commonly use PPO~\citep{schulman2017proximal}. More recent verifiable-reward
systems use outcomes that can be checked automatically. DeepSeekMath introduced GRPO
for this setting; T\"ulu~3 includes verifiable-reward training in an open post-training
recipe~\citep{shao2024deepseekmath,lambert2024tulu}. Several studies report that this
kind of post-training can concentrate probability on fewer
outputs~\citep{kirk2024understanding,karouzos2026collapse}. The effect is not limited
to model samples: language-model assistance can also reduce diversity across text
written by people~\citep{padmakumar2024does}. Entropy regularization encourages stochastic
policies~\citep{haarnoja2018soft}, while reference-policy KL control limits drift from
a base model~\citep{jaques2017sequence}. Neither specifies how probability should be
divided among semantic categories.

\textbf{Controllable generation and repeated sampling.} Controllable generation takes a different approach. CTRL conditions the model on
learned control codes~\citep{keskar2019ctrl}, and Plug and Play Language Models guide a
frozen generator with an attribute model~\citep{dathathri2020plug}. Temperature scaling
adjusts the sharpness of the token distribution, while nucleus sampling truncates its
unreliable tail~\citep{holtzman2020curious}. These methods act on one generation at a
time. The distribution across many generations matters in settings such as self-consistency and
repeated-sampling inference~\citep{wang2023selfconsistency,brown2024large}. It also
matters when generated outputs are reused as data: diversity and bias affect synthetic
training sets~\citep{yu2023large}, and the proportions of pretraining domains affect
downstream performance~\citep{xie2023doremi}.

\textbf{Distribution matching for generated text.} The importance of aggregate behavior has also motivated methods that make the target
distribution itself the learning objective.
\citet{khalifa2021distributional} define target text distributions using
user-specified constraints and train a generator to match them.
\citet{korbak2022distribution} distinguish this approach from expected-reward
maximization and use it to limit catastrophic forgetting during language-model
fine-tuning. \citet{jiang2026controlling} study attribute mixtures across repeated model
generations, combining calibrated steering tokens with semantic preference
optimization. The broader distribution-matching literature includes MMD-GAN, which
uses a learned kernel to compare distributions~\citep{li2017mmd}, and GAIL, which
matches expert and policy occupancy measures through an adversarial
objective~\citep{ho2016generative, wilinski2026inverse}. Together, these approaches establish several ways
to control aggregate distributions, including constrained targets, calibrated
steering, kernel matching, and adversarial learning.

\textbf{Groupwise and answer-level objectives.} More recent post-training methods bring aggregate objectives into groupwise policy
optimization. GAPO uses within-group frequencies to promote uniform coverage over
valid outputs~\citep{anschel2025group}. \citet{lochab2026ucpo} likewise address the
indifference of standard verifiable-reward objectives to how probability is distributed
among correct answers, adding an objective that favors uniform coverage of those
answers. SetPO instead assigns credit to a sampled set rather than to each reasoning
path independently~\citep{duan2026setpo}. Distributional Alignment Games derive
rewards over answer marginals through a dual game~\citep{mohri2026dag}, and
\citet{mohri2026generalized} extend this framework to Bregman divergences and
study finite-group estimators for polynomial and logarithmic rewards. These methods are
closest to our setting because they move beyond independent rewards, while our
experiments focus specifically on matching specified categorical proportions alongside
a task reward.

\section{Conclusion}
\label{section:conclusion}

We introduced distribution matching as a post-training objective and \method{} as a GRPO-based method for achieving it. \method{} uses the category frequencies within each rollout group to reward outputs that move the policy toward a specified target
distribution while optimizing the original task.

In the math and programming experiments, \method{} fits both uniform and peaked targets on all three datasets, whereas correctness-only RL shifts the output distribution away from uniform. Relative to the SFT initialization, median holdout pass@1 ranges from $0.2$ percentage points lower to $10.4$ points higher. Relative to correctness-only RLVR, however, it is $1.2$--$8.5$ points lower. Distributional control can therefore preserve or improve the starting model's accuracy, but it carries a measurable tradeoff relative to optimizing correctness.

Empirically we found GAPO style $L_2$ to be the weakest of the choices for distribution matching. It gives the poorest fit in both post-training
experiments, and in the synthetic closed-set choice it is strictly dominated on the fit-versus-validity frontier in two of the four prompt-form $\times$ target cells. It is the best reward in none of them. Task accuracy is similar across all four rewards, so $L_2$ is not trading fit for accuracy; at equal accuracy it simply has the worse fit in terms of distribution matching.

Among the remaining three, the ranking is heavily task-dependent. In the synthetic closed-set
choice, where the valid outputs are known in advance, forward KL matches the target most
closely but has a higher invalid-output rate, while JSD offers the best compromise between fit and validity. Withholding the option list leaves that ordering unchanged but raises both divergence and the invalid-output rate. In the two post-training experiments,
where distribution matching is combined with correctness training, reverse KL matches the peaked target most closely on all three datasets. This reversal suggests that the best divergence for isolated distribution matching need not be the best one when matching is combined with correctness. The choice among these three is therefore a choice about fit, not about accuracy.

We hypothesize that the marginal accuracy improvements stem from the fact that our current experiments use 4B scale model and relatively smaller post-training tasks. At larger model and dataset scales we believe that enhanced diversity would improve exploration resulting in significant improvements in accuracy. We further limit our contributions to the empirical contributions but further research is required to understand the relative differences in the optimization dynamics of the various divergence metrics.

Overall, \method{} provides a simple way to control categorical output distributions during post-training, while making both the trade-off with task accuracy and the choice of divergence explicit.

\clearpage
\newpage
\bibliographystyle{assets/plainnat}
\bibliography{paper}

\clearpage
\newpage
\beginappendix

\section{Proof and Derivation of \method{}}
\label{appendix:method-details}

\noindent\textbf{Signed-gap derivation.} For a fixed prompt $x$, write $\ell_{\mathrm{dist}}(\theta;x)=\tfrac{1}{2}\sum_c(q_\theta(c\mid x)-p_c^\star)^2$. Differentiating the squared error gives
\begin{equation}
-\nabla_\theta\ell_{\mathrm{dist}}(\theta;x)
=
\sum_{c=1}^{C}
\bigl(p_c^\star-q_\theta(c\mid x)\bigr)
\nabla_\theta q_\theta(c\mid x).
\label{eq:distribution-loss-gradient}
\end{equation}
Because $f$ is fixed during the policy update, the score-function identity gives
\begin{equation}
\nabla_\theta q_\theta(c\mid x)
=
\mathbb{E}_{y\sim\pi_\theta(\cdot\mid x)}
\left[
\mathbf{1}\{f(y)=c\}
\nabla_\theta\log\pi_\theta(y\mid x)
\right].
\label{eq:category-probability-gradient}
\end{equation}
Substituting \cref{eq:category-probability-gradient} into
\cref{eq:distribution-loss-gradient} and averaging over prompts gives the signed-gap
gradient that \method{} implements,
\begin{equation}
\begin{aligned}
-\nabla_\theta\mathcal{L}_{\mathrm{dist}}(\theta)
={}& \mathbb{E}_{x,\,y\sim\pi_\theta(\cdot\mid x)} \Bigl[
\bigl(p^\star_{f(y)}-q_\theta(f(y)\mid x)\bigr) \\
&\qquad\quad \times \nabla_\theta\log\pi_\theta(y\mid x) \Bigr],
\end{aligned}
\label{eq:signed-gap-gradient}
\end{equation}
that is, the score function weighted by the signed gap between target and realized
frequency for the completion's own category.

\noindent\textbf{Reward forms.}
\label{appendix:reward-forms}
The distributional term is the signed gap between the target and the realized
frequency for a completion's own category,
\begin{equation}
r^{\mathrm{div}}_i \;=\; p^{*}_{c_i} - \widehat q_{c_i}.
\label{eq:rdiv}
\end{equation}
Two reward forms combine it with the correctness reward. \emph{Additive} takes a convex
combination, so every completion is pushed toward the target whether or not it is
correct,
\begin{equation}
r^{\mathrm{add}}_i \;=\; \alpha\, r^{\mathrm{corr}}_i + (1-\alpha)\, r^{\mathrm{div}}_i ,
\label{eq:add}
\end{equation}
where $\alpha$ is the correctness weight. \emph{Correctness-gated} instead lets the
distributional term modulate only the reward of correct completions,
\begin{equation}
r^{\mathrm{gated}}_i \;=\; r^{\mathrm{corr}}_i \bigl(1 + \lambda_{\mathrm{g}}\, r^{\mathrm{div}}_i\bigr),
\label{eq:gated}
\end{equation}
so steering reallocates probability among correct solutions. The four additive settings
are $\alpha \in \{0.90, 0.80, 0.70, 0.55\}$, placing $10$, $20$, $30$ and $45$ percent of
the reward on the distributional term (\cref{section:setup}).

\noindent\textbf{GRPO update.} Let $r_i$ denote either \cref{eq:add} or \cref{eq:gated}. GRPO forms the group-relative advantage
\begin{equation}
\widehat A_i
=
\frac{r_i-\overline r}{s_r+\epsilon},
\qquad
\overline r=\frac{1}{G}\sum_{j=1}^{G}r_j,
\label{eq:group-relative-advantage}
\end{equation}
where $s_r$ is the reward standard deviation within the group and $\epsilon$ a small
constant. The remaining clipping and KL-regularized update is unchanged.

\begin{algorithm}[t]
\caption{One \method{} update}
\label{alg:gdpo}
\begin{algorithmic}[1]
\REQUIRE prompt $x$, rollout policy $\pi_{\theta_{\mathrm{old}}}$,
target $p^\star$, group size $G$, weight $\alpha$ or $\lambda_{\mathrm{g}}$, task evaluator,
category classifier $f$
\STATE Sample $\{y_i\}_{i=1}^{G}$ from
$\pi_{\theta_{\mathrm{old}}}(\cdot\mid x)$
\STATE Compute $r_i^{\mathrm{corr}}$ and $c_i=f(y_i)$ for each response
\STATE Compute $\widehat q_c\gets
G^{-1}\sum_i\mathbf{1}[c_i=c]$ for each category $c$
\STATE Set $r_i^{\mathrm{div}}\gets
p_{c_i}^\star-\widehat q_{c_i}$
\STATE Form $r_i^{\mathrm{add}}$ or $r_i^{\mathrm{gated}}$
\STATE Compute the group-relative advantages $\widehat A_i$
\STATE Apply the standard clipped, KL-regularized GRPO update
\end{algorithmic}
\end{algorithm}

\noindent\textbf{Per-prompt and population estimates.} Within the $G$ responses
sampled for one prompt, the realized frequency of category $c$ is
\begin{equation}
\widehat q_c
=
\frac{1}{G}\sum_{i=1}^{G}\mathbf{1}\{f(y_i)=c\}.
\label{eq:empirical-category-frequency}
\end{equation}
\Cref{eq:empirical-category-frequency} is a per-prompt estimate. For population-level matching, the counts can instead be pooled over a minibatch of $B$ prompts:
\begin{equation}
\widehat q_c^{\mathrm{pop}}
=
\frac{1}{BG}\sum_{b=1}^{B}\sum_{i=1}^{G}
\mathbf{1}\{f(y_{b,i})=c\}.
\label{eq:population-frequency}
\end{equation}
The reward remains $p_{c_i}^\star-\widehat q_{c_i}$, using $\widehat q_c^{\mathrm{pop}}$ in place of the per-prompt estimate.

\section{Synthetic Choice: Full Results}
\label{appendix:closedset}

\Cref{table:closedset} gives the per-objective numbers behind
\cref{section:exp-closed}, for both prompt forms and on both the training pool and the
held-out split. The two
measures are never combined: JSD to target is computed over valid answers only, and
off-support is the share of answers falling outside the declared set. JSD is reported
raw, in nats, so it does not reach $0$ even for an on-target policy: at $K{=}5$ and
$G{=}32$ the finite-sample floor is $0.017$ nats on the uniform target and $0.008$ on
the peaked one. Those floors are common to all four rewards and do not affect any
comparison below.

\begin{table}[t]\centering\small
\setlength{\tabcolsep}{5pt}
\begin{tabular}{llccc}
\toprule
 & & train & \multicolumn{2}{c}{held-out} \\
\cmidrule(lr){3-3}\cmidrule(lr){4-5}
target & reward & JSD & JSD & off-support \\
\midrule
\multicolumn{5}{l}{\textbf{options shown} (main text)} \\
\multicolumn{5}{l}{\quad\emph{uniform}} \\
 & $L_2$ & $0.0510$ & $0.0901$ & $0.0070$ \\
 & KL$(p\Vert q)$ & $0.0265$ & $0.0467$ & $0.0189$ \\
 & KL$(q\Vert p)$, $\varepsilon$-sm. & $0.0263$ & $0.0504$ & $0.0135$ \\
 & JSD & $0.0388$ & $0.0720$ & $0.0062$ \\
\multicolumn{5}{l}{\quad\emph{truncated}} \\
 & $L_2$ & $0.0706$ & $0.1999$ & $0.0078$ \\
 & KL$(p\Vert q)$ & $0.0201$ & $0.1467$ & $0.0309$ \\
 & KL$(q\Vert p)$, $\varepsilon$-sm. & $0.0564$ & $0.1846$ & $0.0401$ \\
 & JSD & $0.0411$ & $0.1910$ & $0.0103$ \\
\midrule
\multicolumn{5}{l}{\textbf{options inferred}} \\
\multicolumn{5}{l}{\quad\emph{uniform}} \\
 & $L_2$ & $0.1465$ & $0.2053$ & $0.1125$ \\
 & KL$(p\Vert q)$ & $0.0451$ & $0.1413$ & $0.1814$ \\
 & KL$(q\Vert p)$, $\varepsilon$-sm. & $0.0633$ & $0.1424$ & $0.1876$ \\
 & JSD & $0.1176$ & $0.2015$ & $0.1130$ \\
\multicolumn{5}{l}{\quad\emph{truncated}} \\
 & $L_2$ & $0.1893$ & $0.4294$ & $0.1065$ \\
 & KL$(p\Vert q)$ & $0.0598$ & $0.3107$ & $0.2438$ \\
 & KL$(q\Vert p)$, $\varepsilon$-sm. & $0.1546$ & $0.4141$ & $0.1856$ \\
 & JSD & $0.1586$ & $0.4139$ & $0.0979$ \\
\bottomrule
\end{tabular}
\caption{\textbf{Closed-set results, mean over three seeds, $K{=}5$.} JSD to target in
nats over valid answers; lower is better. Train JSD is the mean of the last ten logged
points; held-out is the final evaluation on $80$ prompts from $8$ topics disjoint from
training. Off-support is the held-out rate of answers outside the declared set; it is
not logged on the training pool. Reverse KL is $\varepsilon$-smoothed: its coefficient
is unbounded where the target assigns zero mass, so on the truncated target its value
is set partly by the floor $\varepsilon$. The main text reports the options-shown rows;
the options-inferred rows are discussed in \cref{appendix:closedset}.}
\label{table:closedset}
\end{table}

\paragraph{Tradeoff of withholding the option list.}
The lower half of \cref{table:closedset} gives the same per-objective numbers for the
open-set form, in which the clause listing the valid answers is removed and the model
must infer the answer set as well as spread over it. Nothing else changes: same topics,
same targets, same rewards, same three seeds. Inferring the set has more downsides than any divergence choice. Mean held-out JSD rises from $0.0648$ to $0.1726$ on the uniform
target and from $0.1806$ to $0.3920$ on the peaked one, factors of $2.7$ and $2.2$, and
off-support rises by more than an order of magnitude: the model that must guess the set
frequently guesses outside it, and unlike the shape term this barely improves over
training. The ranking across the four objectives is unchanged
(\cref{section:exp-closed}), which is why we report it as a property of the reward
rather than of the prompt. The two-objective picture is preserved and slightly sharper:
JSD lies on the fit-versus-validity frontier in all four (prompt form $\times$ target)
cells, while $L_2$ is strictly dominated in two of them, here by JSD on the peaked
target ($0.4139$ against $0.4294$ JSD and $0.0979$ against $0.1065$ off-support).

\section{Post-Training Experiments: Setup and Full Grid}
\label{appendix:posttrain}

\subsection{Experimental setup}
\label{appendix:math-setup}
Three datasets share one design. GSM8K and MATH use the natural language of the solution
as the categorical axis, over English, Chinese, Spanish, Hindi and French; CodeContests
uses the programming language of the solution, over Python, C++, Java, Go and JavaScript.
Each dataset has its own SFT init, training pool and held-out evaluation
set; nothing is shared between them except the code that trains and scores them.

\paragraph{Initialization.}
The base policy is Qwen3-4B-Instruct-2507 in all three. Left alone it puts every categorized completion in one
category. Each dataset is therefore first given a short multilingual supervised finetune,
built by rejection sampling correct solutions in each of the five categories. Its only purpose is coverage: it must make all five reachable under sampling
without the prompt naming one. The checkpoint is chosen on a held-out selection split disjoint from both the RL training
pool and the evaluation set, and it is chosen on category coverage rather than accuracy.
The accuracy-optimal checkpoint answers everything in the strongest category, which is the
collapsed policy the method exists to move away from.

\paragraph{Reward.}
The distributional term is the score-function coefficient of one of the four divergences
of \cref{section:method}, computed over the categories present in a group. It combines
with correctness either additively, $\alpha r^{\mathrm{corr}}+(1-\alpha)r^{\mathrm{div}}$, or
correctness-gated, $r^{\mathrm{corr}}(1+\lambda_{\mathrm{g}} r^{\mathrm{div}})$
(\cref{appendix:reward-forms}). The four additive settings place $10$, $20$, $30$ and $45$
percent of the reward on the distributional term. A rollout whose category falls outside
the declared set carries a flat validity penalty.

\paragraph{Training and evaluation.}
Training is $1200$ GRPO steps, group size $16$, learning rate $5\times10^{-6}$, per-token
KL penalty $0.04$ and temperature $1.0$. Evaluation samples $k{=}4$ completions
for each of $200$ held-out problems, $800$ per run, at temperature $1.0$ with no
nucleus or top-$k$ truncation, from a prompt that names no category. Correctness on GSM8K
and MATH is exact match after answer extraction; on CodeContests it is execution against
the full test suite, an average of $124.7$ tests per problem. Category shares are
renormalized over the five categories, so off-support completions are excluded from the
divergence and reported separately. Every configuration is run at three seeds
($42$, $123$, $200$), giving $123$ runs per dataset: four divergences $\times$ five reward
configurations $\times$ two targets $\times$ three seeds, plus three RLVR controls.

\paragraph{The estimator floor.}
JSD computed from a finite sample is biased upward, so it does not reach $0$ even for an
on-target policy. Simulating the estimator at the true target gives a mean of $0.00063$
nats against the uniform target and $0.00032$ against the peaked one. Differences within
a small factor of these are not separable, which is why \cref{section:exp-realistic}
reads the uniform column as a check rather than as a comparison.
\newpage
\subsection{Full reward grid}
\Cref{table:rl-grid} presents the full tabulated rewards corresponding to the experiment grid.

\label{appendix:rl-grid}
\begin{table*}[!ht]
    \centering
    \small
    \begin{NiceTabular}{llcccccc}
    \toprule
    & & \multicolumn{3}{c}{uniform target} & \multicolumn{3}{c}{peaked target} \\
    \cmidrule(lr){3-5}\cmidrule(lr){6-8}
    Reward & Weight & GSM8K & MATH & CodeContests & GSM8K & MATH & CodeContests \\
    \midrule
    \texttt{l2} & \texttt{gated} & $0.0033$ & $0.0081$ & $0.0442$ & $0.0267$ & $0.0281$ & $0.1056$ \\
     & \texttt{add90} & $0.0190$ & $0.0043$ & $0.0615$ & $0.0236$ & $0.0300$ & $0.2178$ \\
     & \texttt{add80} & $0.0023$ & $0.0067$ & $0.0303$ & $0.0210$ & $0.0300$ & $0.0819$ \\
     & \texttt{add70} & $0.0027$ & $0.0073$ & $0.0177$ & $0.0681$ & $0.0280$ & $0.1110$ \\
     & \texttt{add55} & $0.0279$ & $0.0047$ & $0.0285$ & $0.0617$ & $0.0237$ & $0.0554$ \\
    \addlinespace[2pt]
    \texttt{fkl} & \texttt{gated} & $0.0049$ & $0.0045$ & $0.0338$ & $0.0073$ & $0.0163$ & $0.0106$ \\
     & \texttt{add90} & $0.0048$ & $0.0077$ & $0.0070$ & $0.0138$ & $0.0130$ & $0.0285$ \\
     & \texttt{add80} & $0.0037$ & $0.0039$ & $0.0098$ & $0.0137$ & $0.0177$ & $0.0201$ \\
     & \texttt{add70} & $0.0036$ & $0.0044$ & $0.0090$ & $0.0502$ & $0.0126$ & $0.0117$ \\
     & \texttt{add55} & $0.0044$ & $0.0095$ & $0.0061$ & $0.0061$ & $0.0074$ & $0.0056$ \\
    \addlinespace[2pt]
    \texttt{rkl} & \texttt{gated} & $0.0067$ & $0.0095$ & $0.0160$ & $0.0486$ & $0.0104$ & $0.0082$ \\
     & \texttt{add90} & $0.0032$ & $0.0137$ & $0.0396$ & $0.0124$ & $0.0054$ & $0.0100$ \\
     & \texttt{add80} & $0.0028$ & $0.0037$ & $0.0069$ & $0.0045$ & $0.0080$ & $0.0085$ \\
     & \texttt{add70} & $0.0034$ & $0.0050$ & $0.0080$ & $0.0077$ & $0.0095$ & $0.0022$ \\
     & \texttt{add55} & $0.0024$ & $0.0050$ & $0.0025$ & $0.0049$ & $0.0070$ & $0.0036$ \\
    \addlinespace[2pt]
    \texttt{jsd} & \texttt{gated} & $0.0018$ & $0.0115$ & $0.0679$ & $0.0114$ & $0.0253$ & $0.1229$ \\
     & \texttt{add90} & $0.0035$ & $0.0105$ & $0.0848$ & $0.0083$ & $0.0159$ & $0.0445$ \\
     & \texttt{add80} & $0.0063$ & $0.0054$ & $0.0173$ & $0.0530$ & $0.0105$ & $0.0303$ \\
     & \texttt{add70} & $0.0244$ & $0.0101$ & $0.0177$ & $0.0455$ & $0.0105$ & $0.0276$ \\
     & \texttt{add55} & $0.0059$ & $0.0047$ & $0.0112$ & $0.0529$ & $0.0153$ & $0.0124$ \\
    \bottomrule
    \end{NiceTabular}
    \caption{\textbf{Full reward grid, all three datasets.} JSD to target
    in nats, mean over three seeds, at step $1200$. Rows \texttt{add90} to \texttt{add55}
    are the additive form, the suffix giving the percentage weight on correctness
    ($\alpha$ in \cref{eq:add}); \texttt{gated} is the correctness-gated form. The main-text figure and \cref{table:rl-main} pool the four additive
    rows; the gated row is shown here but excluded there, since it is a different
    combination rule rather than a different mixing weight. The mixing weight moves JSD
    by about as much as the divergence does (median spread across the four weights
    $0.0057$ to $0.0444$ nats against $0.0053$ to $0.0911$ across divergences), which is
    why no single weight is reported as representative.}
    \label{table:rl-grid}
\end{table*}

\newpage
\subsection{Per-category accuracy}
\label{appendix:perlang}
Accuracy within a category is not a clean measure of competence in it, because each run
chooses which problems it answers in which category: a run that reaches for Hindi only on
the easy problems will look better at Hindi. We report it anyway, with two guards. Counts are pooled across the twelve additive runs
behind each reward rather than averaged over per-run rates, so a run that emitted a
category twice does not weigh as much as one that emitted it eighty times. A category is
reported only where the pooled count reaches $30$ completions. Under the peaked target the two excluded categories fall below that
floor for most rewards, which is the target working as specified.

\Cref{fig:rl-perlang} gives the result. The direction is consistent across the three
datasets: the categories the SFT init was weakest in gain, and the strong ones give
ground. On CodeContests, Go rises from $0.151$ at the SFT init to $0.165$--$0.210$
and JavaScript from $0.295$ to between $0.252$ and $0.370$, gaining under seven of the
eight reward and target combinations, while C++ falls from $0.465$ to $0.352$--$0.438$. On GSM8K, Hindi rises from $0.605$ to $0.659$--$0.706$ and no other
category moves by more than six points. On MATH almost every category gains, which is the
accuracy effect noted above.

\begin{figure*}[ht]
    \centering
    \includegraphics[width=\textwidth]{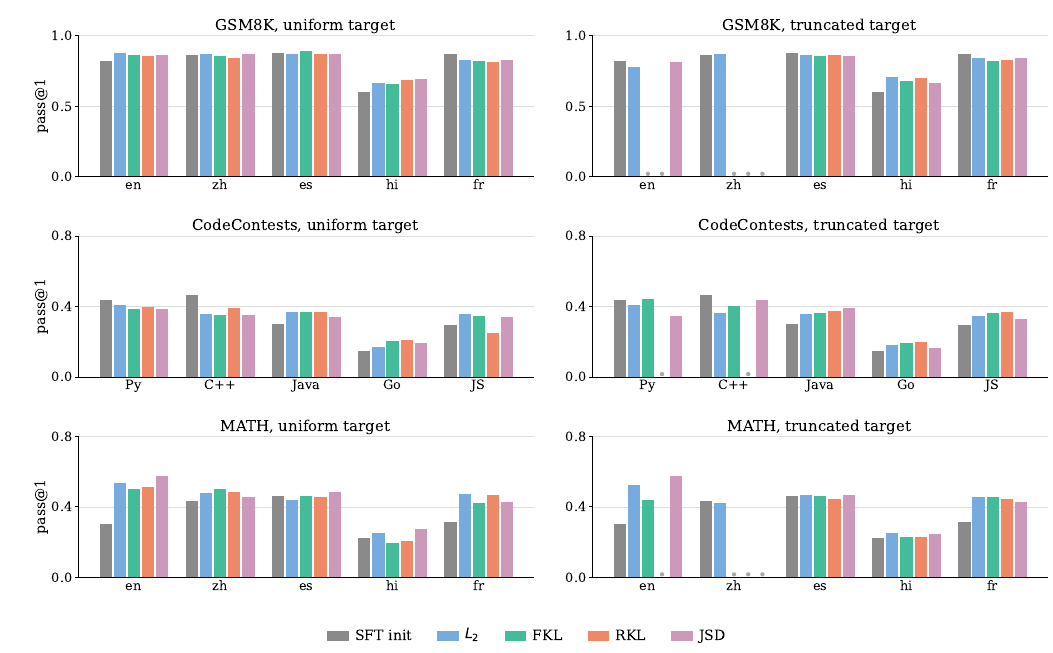}
    \caption{\textbf{Per-category holdout pass@1.} Rows are datasets, columns are targets.
    Each reward bar pools the twelve additive runs behind it by summing correct
    completions and completions. Grey is the SFT init. A dot on the
    baseline marks a category that falls below the $30$-completion reporting floor; it is
    not a measured zero.}
    \label{fig:rl-perlang}
\end{figure*}

\end{document}